\pdfoutput=1
\documentclass[conference]{IEEEtran}
\IEEEoverridecommandlockouts
\usepackage{cite}

\usepackage{amsmath,amssymb,amsfonts,amsthm}
\usepackage[acronym]{glossaries}
\makeglossaries
\usepackage{graphicx}
\usepackage{hyperref}
\usepackage{textcomp}
\usepackage{xcolor}
\usepackage{comment}
\usepackage{booktabs} 
\usepackage{multirow}
\usepackage{rotating}
\usepackage{hyperref}
\newtheorem{proposition}{Proposition}
\newtheorem{definition}{Definition}

\def\Vc{V_{c,i}}
\def\Vsucc{V_{c,s(i)}}
\def\alphai{\alpha_i}

\def\calT{\mathcal{T}}
\def\Cfree{\mathcal{C}_{free}}
\def\simcom{\mathcal{T} = \{\Delta_1, ..., \Delta_N\}}
\def\Vcj{V_{c,j}}
\def\Cinf{C^\infty}
\def\Tilal{\tilde{\alpha}}
\def\fexiti{f_{\text{exit},i}}
\def\dthetad{\dot{\theta}_d}
\def\pxthetad{\frac{\partial \theta_d}{\partial x}}
\def\pythetad{\frac{\partial \theta_d}{\partial y}}
\def\Bsafe{\partial\mathcal{S}}
\def\safe{\mathcal{S}}
\def\free{\mathcal{F}}
\DeclareMathOperator{\atan}{atan2}

\newacronym{qp}{QP}{quadratic programming}
\newacronym{apf}{APF}{artificial potential fields}
\newacronym{cbf}{CBF}{control barrier function}
\newacronym{dt}{DT}{Delaunay triangulation}
\newacronym{cdt}{CDT}{constrained Delaunay triangulation}
\newacronym{lp}{LP}{linear programs}
\newacronym{gvd}{GVD}{generalized voronoi diagram}

\def\BibTeX{{\rm B\kern-.05em{\sc i\kern-.025em b}\kern-.08em
    T\kern-.1667em\lower.7ex\hbox{E}\kern-.125emX}}
\begin{document}

\title{\LARGE Planning Smooth and Safe Control Laws for a\\ Unicycle Robot Among Obstacles 
}
\author{Aref Amiri$^1$, Ba\c{s}ak Sak\c{c}ak$^{1,2}$, and Steven M. LaValle$^1$
\thanks{\copyright~2026 IEEE. This work has been accepted for publication in the 2026 European Control Conference (ECC). Personal use of this material is permitted. Permission from IEEE must be obtained for all other uses.}
\thanks{This work was supported by Infotech Oulu, a European Research Council Advanced Grant, and Academy of Finland (project BANG! 363637).}
\thanks{$^{1}$Faculty of Information Technology and Electrical Engineering, University of Oulu, 
{\tt\small firstname.lastname@oulu.fi}}%
\thanks{$^{2}$Department of Advanced Computing Sciences, Maastricht University {\tt\small basak.sakcak@maastrichtuniversity.nl}}}

\maketitle

\begin{abstract}
This paper presents a framework for safe navigation of a unicycle point robot to a goal position in an environment populated with obstacles from almost any admissible state, considering input limits. We introduce a novel \gls{qp} formulation to create a $\Cinf$-smooth vector field with reduced total bending and total turning. Then we design an analytic, non-linear feedback controller that inherently satisfies the conditions of Nagumo's theorem, ensuring forward invariance of the safe set without requiring any online optimization. We have demonstrated that our controller, even under hard input limits, safely converges to the goal position. Simulations confirm the effectiveness of the proposed framework, resulting in a twice faster arrival time with over 50\% lower angular control effort compared to the baseline.
\end{abstract}


\section{Introduction}
Motion planning is concerned with finding a trajectory connecting an initial state with a goal state that avoids obstacles. However, this solution is open-loop. Therefore, applying (if available) the computed controls or following the computed trajectory may deviate from the goal.
Feedback motion planning avoids this limitation by generating a state-feedback control law over the free configuration space $\Cfree$~\cite{lavalle2006planning,lindemann2009simple}. 
This control law can steer almost any admissible robot configuration to the goal, ensuring also some inherent robustness to disturbances.
Classical methods for computing feedback laws to solve motion planning problems rely on \gls{apf} \cite{khatib1986real}.
\gls{apf} combine attractive and repulsive forces, moving the robot along the gradient of a potential function. However, they suffer from local minima (where attractive and repulsive forces cancel each other) and poor performance in narrow passages \cite{koren1991potential}. Navigation functions\cite{rimon1990exact}, a special class of APFs, overcome local minima; however, they are limited to simple environments. 

Considering determining a feedback law that is globally convergent to the goal, Lindemann and LaValle~\cite{lindemann2009simple} proposed blending smooth vector fields computed over a convex decomposition of $\Cfree$. The main drawback of this approach is that the quality of the resulting integral curves is highly sensitive to the assignment of local vector fields and the underlying cell decomposition. The quality of the vector fields becomes critical when it is used to guide a system, such as a unicycle robot. A guidance vector field with unnecessary bending leads to longer travel time, longer paths, and higher control efforts. This paper builds upon~\cite{lindemann2009simple} and we propose a new formulation to assign cell vector fields with a convex \gls{qp}. 

Because the vector fields are locally assigned, the choice of an appropriate cell decomposition is critical. Convex cell decomposition methods provide a structured way to partition the free configuration space, $\Cfree$, into a set of convex regions. Standard decomposition methods include trapezoidal cell decomposition \cite{deBerg2008} and simplicial decompositions (which our work is built upon) such as \gls{dt}~\cite{lee1980two}.
 
In higher dimensions, where obstacle boundaries are represented implicitly, or non-polygonal obstacles in 2D, volumetric collision detection modules can be used to embed an approximate yet safe triangulation within $\Cfree$ \cite{yershov2015planning}. For an environment with polygonal obstacles, 
\gls{cdt} is a particularly suitable approach as it can produce an exact cell decomposition, which can be done in $O(n \log n)$ \cite{shewchuk1996triangle}. \gls{cdt} is a generalization that respects the obstacle boundaries by forcing the specified polygonal edges into the triangulation \cite{engwirda2014locally}. \gls{cdt}, without the addition of extra vertices (Steiner points), can still produce thin triangles. Although using Steiner points can improve mesh quality, our framework can operate on any simplicial complex embedded in the free position space $\free$.

A successful navigation method should ensure that the robot reaches the goal while avoiding obstacles and respecting additional relevant constraints.
\glspl{cbf} are a powerful tool for providing formal safety guarantees by ensuring that the state is forward invariant within a safe set \cite{ames2016control}. Satisfying a \gls{cbf}-based constraint implies safety, and this is typically imposed through a \gls{qp} that works as a safety filter. The common approach in \glspl{cbf} is to maintain a predetermined safe distance from obstacles \cite{desai2022clf,singletary2021comparative}. These approaches are practical for holonomic vehicles but are often conservative for non-holonomic systems, since the system is unable to avoid obstacles by controlling angular velocity, requiring more complex \glspl{cbf} \cite{huang2023obstacle,lee2025turning}. In contrast to \gls{cbf} methods that typically solve a \gls{qp} at each step to find the control inputs, our analytic controller inherently provides a formal safety guarantee without requiring optimization.

This paper presents a framework for constructing a feedback motion plan for a unicycle robot moving in a 2D workspace populated by obstacles. Our approach is to sequentially determine a guidance vector field over the workspace free from obstacles, and use this guidance vector field to determine a state-feedback control law for the unicycle system. While control of a unicycle to track a predefined trajectory has already been explored in the literature (e.g.,~\cite{lee2001tracking,feedbackllintracking}), planning a feedback control law is inherently more robust. 
We first introduce a novel \gls{qp} formulation to generate the underlying guidance vector field. This produces a $\Cinf$-smooth field with lower curvature compared to \cite{lindemann2009simple}, which is essential for efficiency. Then, we introduce a novel non-linear feedback controller, inspired by the principles in \cite{lindemann2006real}. We demonstrate asymptotic convergence to the goal from almost any admissible state. Furthermore, we demonstrate that our analytic controller inherently satisfies the conditions of Nagumo's theorem~\cite{egerstedt2021robot}, ensuring forward invariance of the safe set. Additionally, we show that our controller ensures that the safe set is globally attractive. Finally, we demonstrate through simulation that the synergy of our low-curvature QP-vector field and the controller allows the robot to reach the goal with a significantly lower arrival time and angular control effort, proving its effectiveness even under hard input saturation.

\section{Problem Formulation}
We consider the problem of safe navigation of a unicycle point robot in a 2D environment populated with obstacles. Let $\mathcal{W}\subseteq\mathbb{R}^2$ be a planar environment and let $\mathcal{O} \subseteq \mathcal{W}$ be the obstacles. The set of points not occupied by obstacles is the free position space $\free = \mathcal{W}\setminus\mathcal{O}$. The robot's configuration at time $t$ is denoted by $q(t)=[x,y,\theta]^T$, where $p=[x,y]^T\in \mathcal{W}$ is its position and $\theta \in S^1$ is its orientation with respect to a global reference frame. The free configuration space is denoted by $\mathcal{C}_{\text{free}} := \mathcal{F}\times S^1$. A configuration $q \in \mathcal{W}\times S^1$ is collision-free if $q \in \mathcal{C}_{\text{free}}$. The kinematic model of the robot is: 
\begin{equation}
\begin{bmatrix} \dot{x} \\ \dot{y} \\ \dot{\theta} \end{bmatrix} =
\begin{bmatrix} \cos\theta & 0 \\ \sin\theta & 0 \\ 0 & 1 \end{bmatrix}
\begin{bmatrix} v \\ \omega \end{bmatrix},
\label{eq:kinematics}
\end{equation}
where $v$ and $\omega$ are the forward linear velocity and angular velocity control inputs, respectively. These inputs are assumed to be within saturation limits: 
\begin{equation}\label{eq:v_lim}
 0\leq v \leq v_\text{max}, 
\end{equation}
\begin{equation}\label{eq:w_lim}
|\omega| \leq \omega_\text{max}.
\end{equation}

Our method is built upon a simplicial complex embedded in the free space, which can be an exact cell decomposition for the polygonal obstacles (with \gls{cdt}) or an approximation for the non-polygonal obstacles. Given any starting state $q_0$ in the embedded triangulation in $\free$ (except for the set of measure zero, which are the vertices of the triangles), and a goal position $p_g = [x_g,y_g]^T \in \free$, the objective is to design a feedback control law that respects the saturation limits at all times and guarantees: 1) Safety (Collision avoidance): The robot's position remains in the free workspace $\free$ for all time $t \ge 0$. 2) Convergence: The robot's position asymptotically converges to the goal position while respecting the input limits. 

Our approach solves this problem in two stages. Following the same strategy as \cite{lindemann2009simple}, we first compute a $\Cinf$-smooth vector field $V(p):\free \to \mathbb{R}^2$ convergent to $p_g$. To improve the quality of the vector field, we formulate a \gls{qp} problem to assign all cell vector fields, complemented by an averaging technique for the exit face vector fields, to have a smoother transition between cells. Then, by the given vector field, we define a target orientation at every point $p \in \free$, denoted by $\theta_d(V(p))$. The controller must not only drive the orientation error $\phi = \theta- \theta_d$ to zero, but also ensure safety and asymptotic convergence to $p_g$.

\section{Smooth Vector Field Formulation}\label{sec:vector_field_formulation}
The core of our method is the well assignment of local vector fields across a simplicial decomposition embedded in the free space, $\free$. 
This section details our optimization-based approach, which begins by reestablishing the foundational concepts upon which our formulation is based.

\subsection{ Discrete Plan and Vector Field Constraints}

Let $\simcom$ be a simplicial complex embedded in $\free$. To guide the robot motion globally, we first compute a high-level discrete plan over $\mathcal{T}$. This is achieved by first constructing a connectivity graph where nodes correspond to the centroids of the $2$-simplexes (triangles) and edges connect adjacent $2$-simplexes. This construction results in a sparse graph with at most $3N/2$ edges (undirected), where $N$ is the finite number of simplexes embedded in $\free$. A single-source shortest path algorithm, such as Dijkstra's, can then be used to compute a successor mapping, $s(i)$, for each simplex $\Delta_i \in \mathcal{T}\setminus \Delta_G$, in which $\Delta_G$ is the $2$-simplex containing the goal position. This mapping defines the unique successor $2$-simplex that the robot should transition to from $\Delta_i$ to eventually reach the goal simplex $\Delta_G$.

The foundation of our feedback plan is the assignment of a constant cell vector field, $\Vc$, to each $2$-simplex $\Delta_i \in \calT$. Each $2$-simplex has three vertices, and each $1$-dimensional face (edge), like the exit face (the face of a cell shared with its successor), has two vertices. Therefore, for a given exit face, there exists an opposite vertex ($v_{ov,i}$) that does not lie on the exit face. For a $2$-simplex $\Delta_i$, the conical region (a wedge) is defined by two boundary vectors, $\{b_{i,1},b_{i,2}\}$, which point from $v_{ov,i}$ to each of the two vertices of the exit face $f_{\text{exit},i}$. See Figure~\ref{fig:concept} for an illustration. Choosing a cell vector field $\Vc$ pointing inside this conical region guarantees that the robot will go out from the exit face and satisfies the conditions of the definition of the cell vector field \cite{lindemann2009simple}. We recall this definition here.
\begin{definition}
\label{def: def1}
For a $2$-simplex $\Delta_i$ with an exit face $f_{\text{exit},i}$, a cell vector field $V_{c,i}$ is a smooth unit vector field on $\Delta_i$ that satisfies three conditions:
\begin{enumerate}
    \item For each point $p \in \Delta_i$, there exists a $q \in f_{\text{exit},i}$ and $\alpha \in \mathbb{R}$ such that $V_{c,i}(p) = \alpha(q - p)$.
    \item Let $h$ be a \gls{gvd} \cite{88035} face with normal vector $n$. If $V_{c,i}(p) \cdot n = 0$ for some $p \in h$, then $V_{c,i}(p) \cdot n = 0$ for all $p \in h$.
    \item The directed transition graph induced by this choice of vector fields is acyclic, and every path through this graph terminates at the node corresponding to the exit face.
\end{enumerate}
\end{definition}

\begin{proposition}
\label{thm:thm1}
A constant cell vector field $\Vc$ that is chosen to point within the conical region defined by the boundary vectors satisfies the conditions of Definition~\ref{def: def1}.
\end{proposition}
\begin{proof}
 As established in~\cite{11478294}, any positive linear combination of two boundary vectors $\{b_{i,1},b_{i,2}\}$ used to construct $\Vc$ will point toward the exit face $f_{\text{exit},i}$. Since both $\Vc$ and $n$ (the constant normal vector to a face of the \gls{gvd}) are constant vectors, if $\Vc \cdot n = 0$ for some point on $h$, it is zero for all points on $h$. Because $\Vc$ points towards the exit face, any trajectory starting in $\Delta_i$ will proceed monotonically toward $f_{\text{exit},i}$ and is guaranteed to intersect it. For the full proof, we refer the reader to \cite{11478294}.
\end{proof}

\begin{figure}[t]
    \centering
    \includegraphics[width=0.60\columnwidth]{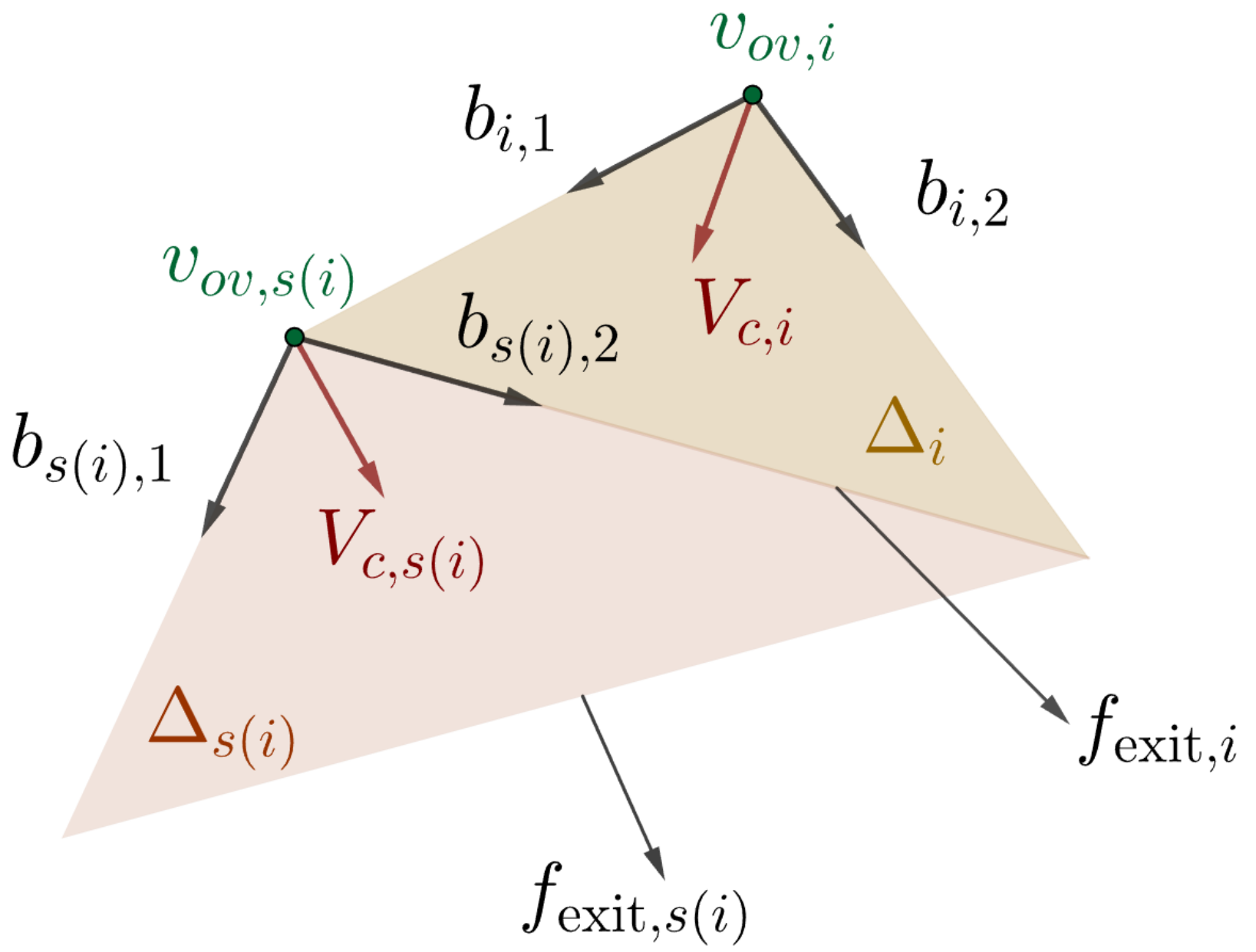}
    \caption{Figure shows two adjacent $2$-simplexes (triangles), $\Delta_i$ and its successor $\Delta_{s(i)}$. For cell $\Delta_i$, the candidate vector field $\Vc$ must lie within the conical region formed by the boundary vectors, $b_{i,1}$ and $b_{i,2}$. These vectors are defined by the opposite vertex $v_{ov,i}$ and the vertices of the shared exit face $\fexiti$.}
    \label{fig:concept}
\end{figure}

\subsection{Assigning Cell Vector Fields via \gls{qp} }

Our objective is to minimize the sum of squared Euclidean differences between the cell vector fields of each cell $i$ and its successor $s(i)$, which encourages directional alignment. For a $2$-simplex $\Delta_i$, the cell vector $\Vc$ is parameterized by two variables $\alphai = [\alpha_{i,1},\alpha_{i,2}]^T$:
\begin{equation}
\Vc(\alphai) = \sum_{k=1}^{2} \alpha_{i,k} b_{i,k}, 
\quad 
\text{s.t.} \quad 
\sum_{k=1}^{2} \alpha_{i,k} = 1, \ \alpha_{i,k} \ge 0.
\end{equation}
To express this in terms of a single independent variable, we eliminate one variable, e.g., $\alpha_{i,2} = 1 - \alpha_{i,1}$. Let $\Tilal_i$ be a vector of these independent variables. The cell vector field can then be:
\begin{equation}
\Vc(\Tilal_i) = b_{i,2} + D_i \Tilal_i,
\end{equation}
where $D_i$ is defined as:
\begin{equation}\label{eq:Di}
D_i = [b_{i,1}-b_{i,2}].
\end{equation}
We then pose the optimization problem as: 
\begin{equation}
\begin{aligned}
\underset{\Tilal_1, \dots,\Tilal_{N-1}}{\text{minimize}} & \quad J = \sum_{i=1}^{N-1} \| \Vc(\Tilal_i) - \Vsucc(\Tilal_{s(i)}) \|^2\\
\text{s.t. } & \hspace{0.8em} 0 \le \Tilal_i\le 1, \forall i = 1, \dots ,N-1,
\end{aligned}
\end{equation}
where $N-1$ is the number of cells in the discrete plan so that the successor index $s(i)$ of cell $i$ is defined. It is critical to note that we minimize the squared vector distance, not the angle directly. This objective is chosen specifically because it results in a convex \gls{qp} problem. This choice serves as a reasonable strategy for aligning the vectors because our parameterization, a convex combination of normalized boundary vectors, keeps the magnitudes of the cell vectors $\|\Vc\|$ within an acceptable, non-zero range. This prevents the optimization from finding the trivial solution of shrinking all the cell vectors to zero and ensures the dominant factor in minimizing the objective is the directional alignment between the vectors. 

Substituting $D_i$ given by Eqn.~\eqref{eq:Di} into a single term of the objective function for an adjacent pair $(i,j)$ (where $j = s(i)$) yields:
\begin{equation}
\|\Vc - \Vcj\|^2 
= \| D_i \Tilal_i - D_j \Tilal_j + c_{ij} \|^2,
\end{equation}
where $c_{ij} = b_{i,2} - b_{j,2}$. Expanding the squared norm yields the standard \gls{qp} form, $\frac{1}{2}\Tilal^T H \Tilal + f^T \Tilal + \text{const}$, where the global Hessian $H$ and linear vector $f$ are assembled by summing the contributions from each pair $(i,j)$. The contribution to the Hessian from the pair $(i,j)$ is a block matrix:
\begin{equation}
H_{ij} = 2
\begin{bmatrix}
D_i^T D_i & -D_i^T D_j \\
-D_j^T D_i & D_j^T D_j
\end{bmatrix},
\end{equation}
and the linear term is:
\begin{equation}
f_{ij} = 2
\begin{bmatrix}
D_i^T c_{ij} \\[3pt] -D_j^T c_{ij}
\end{bmatrix}.
\end{equation}

The global Hessian matrix $H$ is constructed as a sum of matrices of the form $D^T D$, where $D = [D_i, -D_j]$, which is positive semidefinite. Since the sum of positive semidefinite matrices is also positive semidefinite, $H$ is positive semidefinite. Therefore, the problem is a convex \gls{qp}. After finding the cell vector fields, we normalize them to have a unit length. 

\subsection{Assignment of Face Vector Fields}
With the constant cell vector fields, $\Vc$, optimally assigned by the \gls{qp}, the remaining components are the face vector fields. To ensure smooth and safe transitions across boundaries, the constant face vector fields are assigned according to the following rules:
\begin{itemize}
    \item For any non-exit face of a $2$-simplex $\Delta_i$, the face vector is defined as the normalized inward-pointing normal, similar to \cite{lindemann2009simple}. This ensures integral curves to move away from boundaries. 
    \item For the exit face shared between a $2$-simplex $\Delta_i$ and its successor $s(i)$, the face vector is set to the normalized average of their respective cell vector fields, $(\Vc + \Vsucc)$. This promotes a smooth transition from one cell to its successor.
\end{itemize}
\begin{definition}
\label{def: def2}
A face vector field $V_f(f_{\text{exit},i})$ corresponding to the exit face of a cell $\Delta_i$ is a smooth unit vector field satisfying:
\begin{itemize}
    \item For each point $p \in f$, the vector field $V_f(f_{\text{exit},i})$ satisfies $V_f(f_{\text{exit},i}) \cdot n_x > 0$, where $n_x$ is the outward-pointing normal vector for $f = f_{\text{exit},i}$.
    
\end{itemize}
\end{definition}
\begin{proposition}
The averaged exit face vector field $V_f(f_{\text{exit},i}) = \text{normalize}(\Vc + V_{c,j})$ satisfies the conditions of Definition~\ref{def: def2} for an exit face.
\end{proposition}

\begin{proof} 
We know from Proposition~\ref{thm:thm1} that $\Vc \cdot n_x > 0$. Since $\Vcj$ is chosen to be within the conical region $\mathcal{K}_j$, it points away from $f_{\text{exit},i}$, therefore, $V_{c,j} \cdot n_x \ge 0$.  
The dot product of the sum is $(\Vc + \Vcj) \cdot n_x = (\Vc \cdot n_x) + (\Vcj \cdot n_x)$. The sum of a strictly positive and a non-negative term is strictly positive, and this satisfies the conditions. This is a special case of Proposition~2 in~\cite{11478294} and the proof follows from the proof of Proposition~2.
\end{proof}

\subsection{Final Blended Vector Field}
The final blended vector field, $V(p)$, is a $\Cinf$-smooth field synthesized from the local vector fields. At any point $p$ within a $2$-simplex $\Delta_i$, the vector field is a smooth interpolation between the cell's vector field, $\Vc$, and the vector field of the closest face. This blending is performed using a $\Cinf$-smooth bump function $b(\eta)$ similar to \cite{lindemann2009simple}: 
\begin{equation}
b(\eta) = \begin{cases}
0 & \eta \leq 0, \\
\frac{\lambda(\eta)}{\lambda(\eta) + \lambda(1 - \eta)} & 0 < \eta < 1, \\
1 & \eta \geq 1,
\end{cases}
\end{equation}
in which the auxiliary function $\lambda(\eta)$ is defined as $\lambda(\eta) = (1/\eta) e^{-1/\eta}$.

This final vector field, $V(p)$, serves as the reference for our feedback controller discussed in the following section.

\section{Controller}\label{sec:controller}
The second core of this problem is to design a controller for steering the unicycle robot to a goal position. The smooth vector field over the free workspace $\free \subset \mathbb{R}^2$, that is, $V : \free \rightarrow \mathbb{R}^2$, computed according to the previous section, will be used to determine the control inputs to the unicycle system given in Eq.~\eqref{eq:kinematics}.  

The controller is defined by two laws, one for determining the linear velocity $v$ and one for the angular velocity $\omega$. At any point $p \in \free$, the vector field $V$ expresses an instantaneous desired orientation $\theta_d(V(p)) = \atan(V_y(p), V_x(p))$ for the robot where $V(p) = [V_x(p), V_y(p)]^T$. The primary objective of the controller is to drive the orientation error, $\phi = \theta-\theta_d$, between the robot's actual orientation $\theta$ and the target orientation $\theta_d$ to zero. We use the wrapping function to ensure $\phi$ is always within the interval $(-\pi, \pi]$.  

\subsection{Linear velocity control}
The control input corresponding to the linear velocity $v$ is a function of the orientation error $\phi$. The robot should slow down or stop when misaligned to prioritize reorientation, thereby ensuring safe navigation. This behavior can easily be achieved with one of the following laws:
\begin{equation} 
v = v_\text{max} \max\left(0, 1 - \frac{|\phi|}{\epsilon_v}\right),
\label{eq:v_control1}
\end{equation}
or 
\begin{equation} 
v = v_\text{max} \cos(\phi),
\label{eq:v_control2}
\end{equation}
where $v_\text{max}$ is the maximum linear velocity and $\epsilon_v>0$ is a constant corresponding to the angular error threshold. If $|\phi| \ge \epsilon_v$, the robot stops moving forward. The first law restricts the robot to non-negative velocity and provides a tunable response, and the resulting velocity command is continuous. The second law allows for backward motion, where the robot might reverse to better approach its target orientation, therefore the velocity input is restricted to be within $-v_\text{max}\leq v \leq v_\text{max}$. The choice between these two laws represents a design trade-off between a more conservative, forward-only strategy and a smoother, potentially more agile strategy that allows backward motion. 

\subsection{Angular Velocity Control}
The angular velocity $\omega$ is designed as a stabilizing state-feedback law that ensures the robot's orientation $\theta$ smoothly converges to the target orientation $\theta_d(V(p))$. The control law is composed of two components: a feedforward term to anticipate the path's curvature and a feedback term to correct the orientation error. 

The feedforward component, $\dthetad$, calculates the required angular velocity to follow the changing direction of the guidance vector field $V(p)$ as the robot moves. This term allows the robot to steer along the integral curves of the vector field rather than constantly reacting to errors. This concept is analogous to the control required to remain on the ``Target Manifold'' as described in \cite{lindemann2006real}. Using the chain rule $\dthetad$ is expressed as:
\begin{equation}\label{eq:chain_thetad}
   \dthetad = \frac{d\theta_d}{dt} = \pxthetad\dot{x} + \pythetad\dot{y}.
\end{equation}
By substituting the robot's kinematics in \eqref{eq:kinematics} into \eqref{eq:chain_thetad} we get
\begin{equation}   
 \dthetad = v \left( \left. \pxthetad \right|_{x=p_x} \cos\theta + \left.\pythetad \right|_{y=p_y}\sin\theta \right).
\end{equation}
The terms corresponding to the partial derivatives, that is, $\pxthetad$ and $\pythetad$, are computed numerically from the underlying guidance vector field.
The feedback component is responsible for driving the orientation error $\phi$ to zero. It ensures that the target orientation is attractive, meaning the robot will always turn to correct any deviation from the guidance vector. This law can be defined using a nonlinear proportional control law:  $ -k \tanh(\phi)$ where $k>0$ is a positive gain parameter that dictates the rate of convergence. The hyperbolic tangent function, $\tanh(\cdot)$, is chosen for its desirable properties: it provides a smooth, nearly linear response for small errors while naturally saturating for large errors. This prevents overly aggressive turning maneuvers.
Combining the feedforward and feedback components gives the angular velocity control: 
\begin{equation}
\omega = \dthetad - k \tanh(\phi).
\label{eq:omega_control}
\end{equation}
Finally, imposing the input limits in \eqref{eq:w_lim}, the control input corresponding to angular velocity becomes:
\begin{equation}
\omega = \sigma(\dthetad - k \tanh(\phi), \omega_\text{max}),
\label{eq:omega_control_sat}
\end{equation}
where $ \sigma(u, u_\text{max}) $:
\begin{equation}
\sigma(u, u_\text{max}) = \begin{cases}
  u_\text{max} & \text{if } u > u_\text{max}, \\
  u & \text{if } |u| \le u_\text{max}, \\
  -u_\text{max} & \text{if } u < -u_\text{max}.
  \end{cases}
\end{equation}
\section{Stability and Safety Analysis}\label{sec:Stability}
In this section, we provide the formal guarantees for the proposed control system. We demonstrate two key properties: (i) \emph{Safety}, that is, the robot is guaranteed to remain collision-free, (ii) \emph{convergence}, that is, the robot position is guaranteed to converge to a goal position $p_g \in \free$.

Let $V$ be a vector field over $\free$, determined according to the formulation given in Section~\ref{sec:vector_field_formulation}. The closed-loop control system is obtained by substituting the state-feedback controls computed based on $V$, given by Eqs. \eqref{eq:v_control1} and \eqref{eq:omega_control_sat} with $\epsilon_v \in (0, \pi/2]$ (or given by Eqs. \eqref{eq:v_control2} and \eqref{eq:omega_control_sat}), into Eq. \eqref{eq:kinematics}.

\begin{proposition}
\label{prop:safety}
Let $q(t)$ be the trajectory of the closed-loop system initialized at $q(0)=q_0$ for some initial configuration $q_0 \in \mathcal{C}_{\text{free}}$. Then, $q(t)$ satisfies for all $t\geq0$ that $q(t)\in \mathcal{C}_{\text{free}}$ for almost any $q_0$ (other than a set of measure zero).
\end{proposition}

\begin{proof}
By construction of the vector field $V$, $V(p)$ points away from the obstacle for any point $p$ on the obstacle boundary, as the non-exit face vector fields are assigned as inward-pointing normal. Let $V_{\text{rob}}=[v \cos\theta, v \sin\theta]^T$ be the velocity vector of the robot. A collision can only occur if the forward velocity $v>0$ and $V_{\text{rob}}$ points into an obstacle.

Because $V$ points away from obstacles, this can only happen if the orientation error satisfies $|\phi| > \pi/2$. 
Under Eq. \eqref{eq:v_control1}, since we chose $\epsilon_v \in (0, \pi/2]$, the velocity $v$ is strictly driven to zero at or before $|\phi|$ reaches $\pi/2$. Under Eq. \eqref{eq:v_control2}, this law drives $v$ to zero as $|\phi|$ approaches $\pi/2$, and $v$ becomes negative (moving backward, away from the direction of $\theta$) if $|\phi| > \pi/2$. In both cases, in any state where the robot is oriented such that it could drive into an obstacle, its forward velocity $v$ is driven to zero or becomes negative. This safety guarantee is independent of the angular velocity $\omega$ and its saturation. Saturation on $\omega$ (Eq. \eqref{eq:omega_control_sat}) only affects the rate of reorientation; it does not change the fact that $v$ becomes zero or negative, preventing forward motion into an obstacle. Therefore, the system is proven to be collision-free.
\end{proof}

\begin{proposition}
\label{prop:stability}
The trajectory of the closed-loop system initialized at $q(0)=q_0$, governed by the control law given by the Eqs. \eqref{eq:v_control1} and \eqref{eq:omega_control} (or by Eqs. \eqref{eq:v_control2} and \eqref{eq:omega_control}), is guaranteed to converge to the goal region $\mathcal{C}_{\text{goal}}:=\{(p, \theta) \in \mathcal{C}_{\text{free}} \mid  p=p_g\}$ for almost any $q_0$ (other than a set of measure zero).
\end{proposition}
\begin{proof}
First, we prove that the robot's orientation $\theta$ converges to the guidance direction $\theta_d$. Second, we show that this alignment drives the robot's position $p$ to the goal $p_g$. We define the Target Manifold 
$\mathcal{M} = \{ (p, \theta) \in \mathcal{C}_{\text{free}} \mid \phi = \theta - \theta_d(V(p)) = 0  \}$
as the set of all states where the robot is perfectly aligned with the guidance vector field.

We choose a Lyapunov candidate function $L = \frac{1}{2}\phi^2$,  which is a positive definite function with respect to $\phi$ and is zero only on the target manifold. By Eq.~\eqref{eq:omega_control}, the time derivative of $L$ is $\dot{L} = -k \phi \tanh(\phi)$. Since $k > 0$, and $\phi$ and $\tanh(\phi)$ always have the same sign, $\dot{L} \le 0$ and is zero only when $\phi=0$. By the Lyapunov Stability Theorem\cite{khalil2015nonlinear}, the system state is guaranteed to converge to $\phi=0$ and $\theta_d$ is globally asymptotically stable.

Once the robot configuration is on the target manifold $\mathcal{M}$, we have $\phi=0$ and the linear velocity $v$ (from either Eq. \eqref{eq:v_control1} or \eqref{eq:v_control2}) is $v=v_\text{max}$ if $p \ne p_g$. The robot's motion is therefore $\dot{p} = \frac{v_\text{max}}{||V(p)||} V(p)$, and by our construction of the vector fields, $\|V(p)\|=1$. This means the robot's positional trajectory $p(t)$ is equivalent to following the integral curves of the vector field $V(p)$ (scaled by a positive factor $v_\text{max}$). The guidance vector field $V(p)$ is, by construction, globally convergent to the goal $p_g$. Therefore, as the robot follows these integral curves, its position $p(t)$ is guaranteed to converge to the goal $p_g$ (detailed proofs in \cite{lindemann2009simple}).
\end{proof}
\vspace{0.75em}

In Proposition~\ref{prop:stability}, we proved that the robot can converge to the goal state $p_g$ using unsaturated angular velocity controllers. Here we discuss convergence to $p_g$ governed by the saturated control laws in Eq. \eqref{eq:v_control1} and Eq. \eqref{eq:omega_control_sat}. Consider the case when $\phi$ is large; therefore, the robot is misaligned. By Eq. \eqref{eq:v_control1}, $v$ is small or zero. The feedforward $\dthetad$ is a function of $v$; therefore, it is consequently small and the feedback term is dominant. The feedback term, even under saturation, reduces $\phi$ and the system correctly prioritizes reorientation.
Now, if $\phi$ is a small value, the feedback law is consequently small, and saturation can happen when  $\dthetad$ is large. This happens when the robot is moving at high speed. In these situations, a mismatch between the commanded and required angular velocity will cause the robot to fall off the target manifold $\mathcal{M}$, and $|\phi|$ will begin to increase. As $|\phi|$ increases, the controller Eq. \eqref{eq:v_control1} reduces $v$. The reduction of $v$ directly reduces the magnitude of $\dthetad$. This creates a self-regulating loop. When $|\phi|$ increases, $v$ decreases, therefore, $\dthetad$ decreases and saturation most likely stops. Once saturation ends, the controller's stability (from Proposition~\ref{prop:stability}) takes over, $\phi$ is driven back to zero, and $v$ increases again.
The robot never gets stuck unless it is in a set of measure zero. It only stops ($v=0$) if $|\phi| \ge \epsilon_v$ or if it reaches the goal $p_g$. Since $\omega$ is always acting to reduce $\phi$, the robot will not remain in a stopped state (unless $p=p_g$). The results in the next section confirm this claim.

While our safety proof in Proposition~\ref{prop:safety} is intuitive, we show that our proposed analytic control laws (Eqs.~\eqref{eq:v_control1} and \eqref{eq:omega_control_sat}) inherently satisfy Nagumo's theorem condition \cite{egerstedt2021robot}. This demonstrates that our work achieves the same rigorous safety guarantee as CBF-QP methods, but avoids the computational overhead of solving an optimization problem at each control step. As we discussed in Proposition~\ref{prop:safety}, non-exit face vector fields are assigned to point inward; therefore, on the boundary of the obstacles (considering polygonal obstacles and for non-polygonal obstacles, the non-exit face itself), points away. Therefore, when the robot is near the obstacles, having $|\phi| > \pi/2$ with a forward motion can cause a collision. By the control law Eq.~\eqref{eq:v_control1}, robot stops the forward motion when $|\phi| \ge \pi/2$, hence we can define a safe set as all states when $|\phi| \le \pi/2$. We prove that our safe set is not only forward invariant, but also globally attractive.
\begin{proposition}
\label{prop:Nagumo}
Let the safe set be $\safe := \{ q \in \free \times S^1 \mid |\phi| \le \pi/2 \}$. The control law given by Eqs. \eqref{eq:v_control1} and \eqref{eq:omega_control_sat} guarantee that $\safe$ is forward invariant (if $q(0) \in \safe$, then $q(t) \in \safe$ for all $t > 0$) and globally attractive (if $q(0) \notin \safe$, but $q(0) \in \Cfree$, there exists a finite time $T<T_\text{max}$ such that $q(T) \in \safe$)).
\end{proposition}
\begin{proof}
We choose $h(q) = \cos(\phi)$ as a barrier function, which is positive within the interior of the safe set and is zero on its boundary $\Bsafe$ (where $|\phi| = \pi/2$) and negative in the unsafe set.
To guarantee that the safe set is forward invariant, by Nagumo's theorem \cite{egerstedt2021robot}, we must show that the system can always be controlled to satisfy $\dot{h} \ge 0$ at the boundary of the safe set $\Bsafe$ (where $h=0$). While $\dot{h} > 0$ ensures that any trajectory touching the boundary is immediately pushed back into the safe set.
On the $\Bsafe$, the linear velocity $v$ becomes zero; therefore, $\dthetad$ is also zero. Therefore, $\dot{h} = -\sin(\phi) \cdot \sigma(- k \tanh(\phi), \omega_\text{max})$ which in both cases where $\phi=\pi/2$ and $\phi= -\pi/2$, with $k>0$, $\dot{h}$ is positive. Therefore, this strictly forces the robot's state into the safe set, proving forward invariance. 

Now, consider a state $q(0)$ in the unsafe set, $q(0) \notin \safe$, meaning $|\phi(0)| > \pi/2$. In this region, by Eq.~\eqref{eq:v_control1}, $v$ and $\dthetad$ are zero and $\dot{\phi} = \sigma(- k \tanh(\phi), \omega_\text{max})$. If $\phi(t) \in [\pi/2, \pi]$, then $\tanh(\phi) > 0$, so even under saturation $\dot{\phi} < 0$. This means $\phi$ decreases towards $\pi/2$. If $\phi(t) \in [-\pi, -\pi/2]$, then $\tanh(\phi) < 0$ so $\dot{\phi} > 0$ and $\phi$ increases towards $-\pi/2$. In all cases, the dynamics in the unsafe set drive $\phi$ toward the boundary of the safe set. Since $\dot{\phi}$ cannot be zero, the robot cannot get stuck in the unsafe set. Therefore, any state in the unsafe set by our controller is driven into the safe set $\safe$. To find the maximum time $T_\text{max}$, we need the minimum absolute rate of change $|\dot{\phi}|_\text{min}$ when the robot is in the unsafe set. We want to reach $|\phi(T)| = \pi/2$. The magnitude of our rate convergence $|\dot{\phi}|$ has a strict lower bound: $|\dot{\phi}| \ge \min(\omega_\text{max}, k \tanh(\pi/2))>0$ since $\omega_\text{max}$ and $k$ are positive values. This guarantees reaching the safe set boundary in finite time $T\leq (|\phi(0)| - \pi/2)/ \min(\omega_\text{max}, k \tanh(\pi/2))$.
\end{proof}

\section{Results}
First, we discuss the proposed \gls{qp} method for constructing vector fields. Our framework is based on the baseline method \cite{lindemann2009simple}. Lindemann and LaValle acknowledge that the quality of the resulting integral curves can be highly sensitive to cell decomposition and the assignment of the local vector fields. Their simple formulation, such as assigning cell vectors to point toward the exit face centroid and assigning inward normal vectors for non-exit faces, and outward normal vector for the exit face, produces integral curves with unnecessary bending. We should acknowledge that the baseline method works over any convex cell decomposition; therefore, post-processing the triangulation \cite{hertel1983fast} to create larger convex cells can improve the results of their method, although this can be time-consuming.

To address these shortcomings, we proposed a \gls{qp}-based method. As shown in Table~\ref{tab:combined_analysis} and Figure~\ref{fig:comparison_VFs}, this approach significantly improves performance. Following the path quality metrics established in \cite{11478294}, we evaluated the total bending $E_B = \int_0^L \kappa(s)^2 ds$, and total turning $E_T = \int_0^L |\kappa(s)| ds$ with curvature $\kappa(s)$, and path length, observing reductions of over 76\%, 64\% and 9\%, respectively.

The ``Win Rate'' metric, which represents a one-to-one comparison of over 1000 integral curves starting from the same initial position, shows that the \gls{qp}-based method outperforms the baseline in over 96\% of cases for path length and total bending. Therefore, our proposed method generates a more efficient vector fields, which are desirable for the control of the unicycle robot. It is essential to note that while our proposed method can significantly reduce the total bending, we cannot guarantee a lower maximum curvature, as a constant cell vector field and an underlying plan may require and induce sharper turns. 

\begin{figure}[ht]
    \centering
    \setlength{\tabcolsep}{0pt} 
    \begin{tabular}{cc}
        \includegraphics[width=0.45\linewidth,trim={2.5cm 1cm 2.5cm 0.5cm},clip]{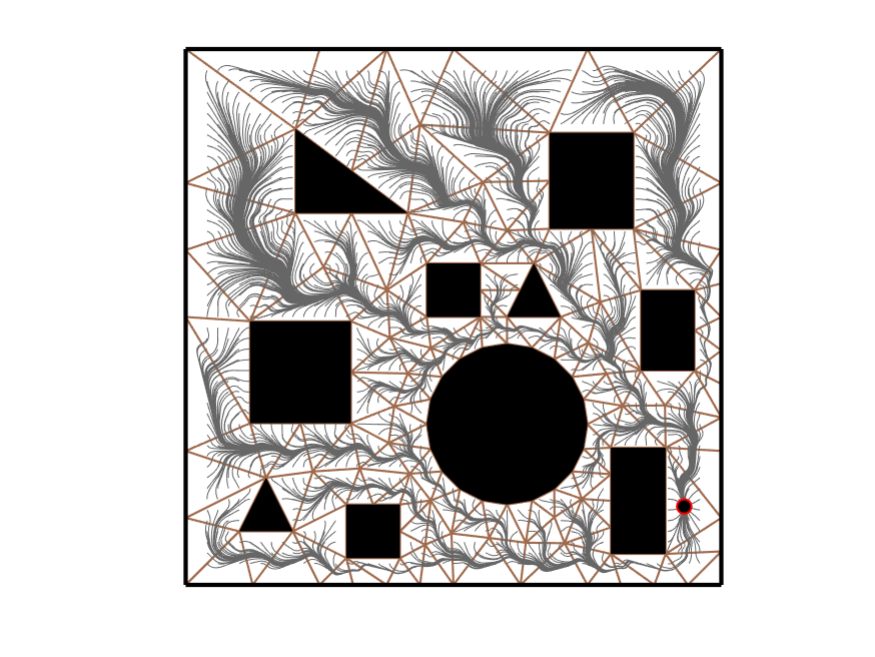} &
        \includegraphics[width=0.45\linewidth,trim={2.5cm 1cm 2.5cm 0.5cm},clip]{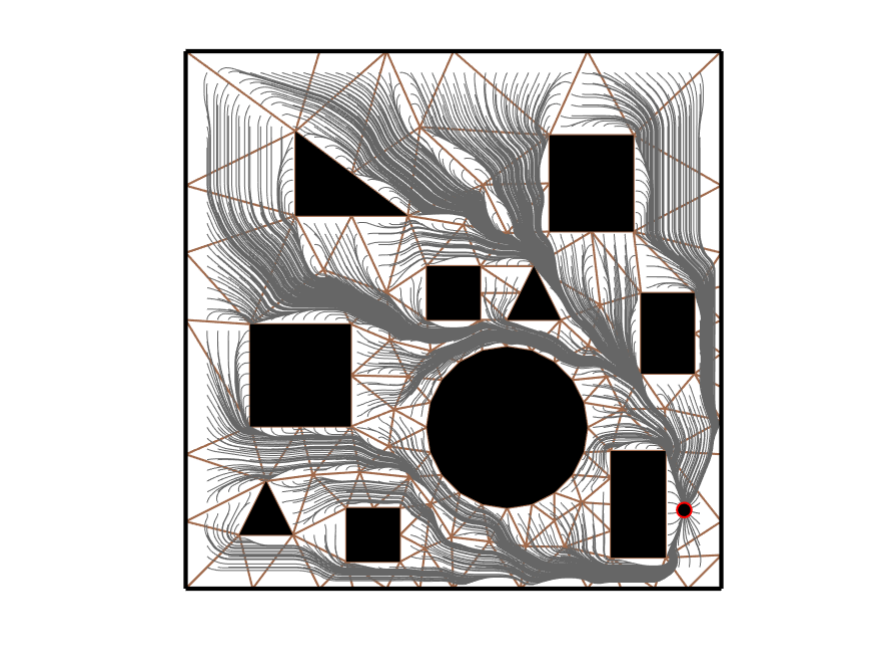} \\[0.1em]
       \includegraphics[width=0.45\linewidth,trim={2.5cm 1cm 2.5cm 0.5cm},clip]{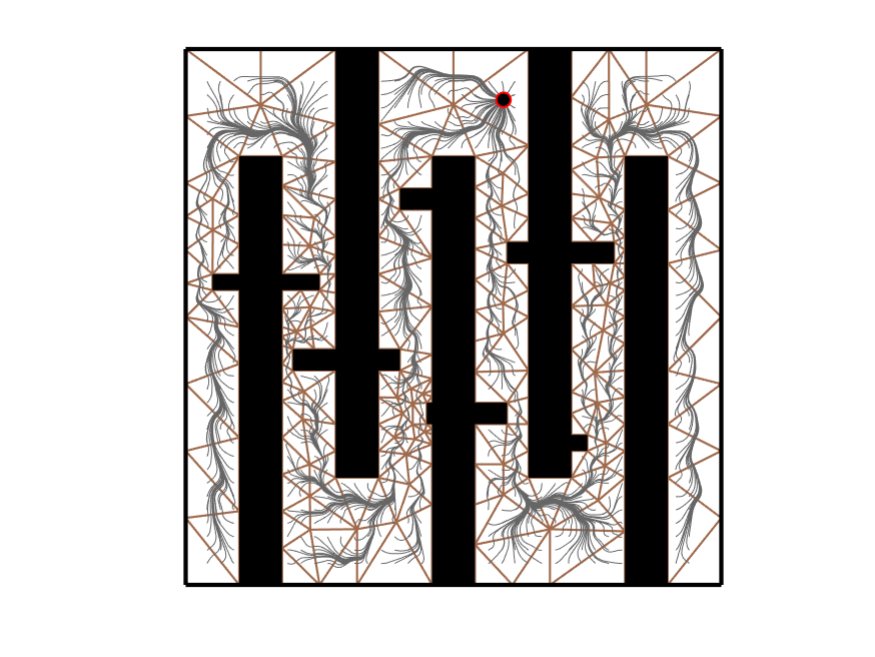} &
        \includegraphics[width=0.45\linewidth,trim={2.5cm 1cm 2.5cm 0.5cm},clip]{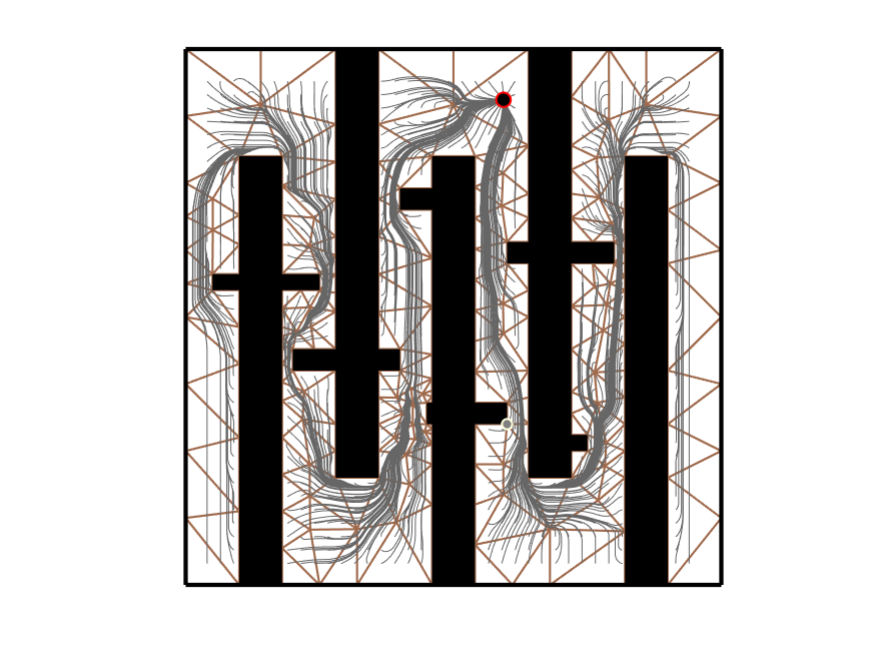}
    \end{tabular}

    \caption{comparison of integral curves for the baseline (left), and proposed \gls{qp} (right) methods across two environments. }
    \label{fig:comparison_VFs}
\end{figure}

\begin{table}[h!]
\centering
\footnotesize 
\caption{Vector Field Quality Comparison. All values are Mean $\pm$ Std Dev.}
\label{tab:combined_analysis}
\resizebox{\columnwidth}{!}{%
\begin{tabular}{@{}l l l r r r@{}}
\toprule
\textbf{Environment} & \textbf{Metric} & \textbf{Method} & \textbf{Value} & \textbf{\% Improv.} & \textbf{Win Rate \%} \\
\midrule
\multirow{12}{*}{\shortstack[l]{\textbf{Cluttered}}} 
 & \multirow{3}{*}{Path Len. (m)} & Baseline & 8.347 ($\pm$3.209) & - & - \\
 & & QP & 7.546 ($\pm$2.929) & 9.6 & 98.91 \\
 \cmidrule(l){2-6}
 & \multirow{3}{*}{Max Curvature} & Baseline & 17.165 ($\pm$7.371) & - & - \\
 & & QP & 17.312 ($\pm$12.879) & -0.85 & 75.26 \\
 \cmidrule(l){2-6}
 & \multirow{3}{*}{Total Bending} & Baseline & 126.507 ($\pm$49.977) & - & - \\
 & & QP & 29.331 ($\pm$29.939) & 76.81 & 96.55 \\
 \cmidrule(l){2-6}
 & \multirow{3}{*}{Total Turning} & Baseline & 21.266 ($\pm$8.015) & - & - \\
 & & QP & 7.509 ($\pm$3.053) & 64.69 & 98.91 \\
\midrule
\multirow{12}{*}{\shortstack[l]{\textbf{Corridor}}} 
 & \multirow{3}{*}{Path Len. (m)} & Baseline & 13.049 ($\pm$7.418) & - & - \\
 & & QP & 11.621 ($\pm$6.406) & 10.94 & 97.80 \\
 \cmidrule(l){2-6}
 & \multirow{3}{*}{Max Curvature} & Baseline & 27.456 ($\pm$11.757) & - & - \\
 & & QP & 27.710 ($\pm$15.279) & -0.93 & 55.69 \\
 \cmidrule(l){2-6}
 & \multirow{3}{*}{Total Bending} & Baseline & 366.344 ($\pm$250.481) & - & - \\
 & & QP & 74.880 ($\pm$55.661) & 79.56 & 96.45 \\
 \cmidrule(l){2-6}
 & \multirow{3}{*}{Total Turning} & Baseline & 43.304 ($\pm$26.574) & - & - \\
 & & QP & 12.52 ($\pm$8.54) & 68.19 & 99.02 \\
\bottomrule
\end{tabular}
} 
\end{table}

To validate the performance of our controller and QP vector field (QP-VF), we conducted two comparative experiments against the baseline vector field (baseline VF). 
\subsection{Experiment 1: Cluttered Environment}

In this work, we adopt the forward-only law given by Eq.~\eqref{eq:v_control1} for all simulations. The robot was simulated in a cluttered environment with input limits of $v_\text{max} = 1.0$~m/s and $\omega_\text{max}= 1.0$~rad/s. The feedback gain was set to $k=1$ and the velocity gain $\epsilon_v = \pi/4$. The initial orientation of the robot was set opposite the direction of the vector field at that point. The results are illustrated in Figure~\ref{fig:res1-2}, Figure~\ref{fig:inputres1}, and Table~\ref{tab:res1}, which show a significant improvement in performance. The proposed vector field produced a shorter path (15.032~m vs 15.261~m), less angular control effort $\int_0^{t_f} \omega(t)^2 dt$ (10.934 vs. 22.619), less time spent in angular saturation, and higher average speed (0.66 vs. 0.496).
\begin{figure}[ht]
    \centering
    \setlength{\tabcolsep}{0pt} 
    \begin{tabular}{cc}
        \includegraphics[width=0.43\columnwidth,trim={11cm 2cm 9cm 1cm},clip]{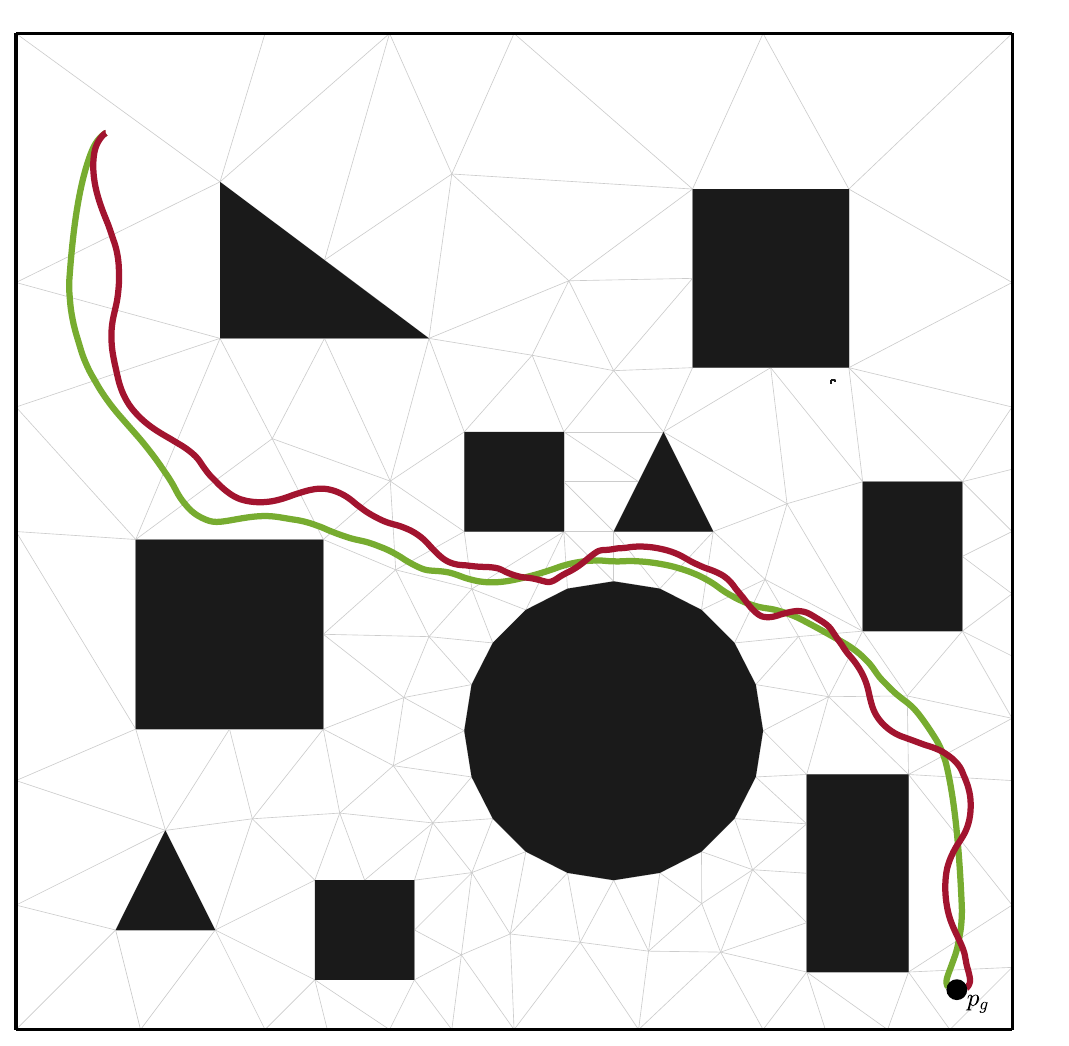} &
        \includegraphics[width=0.57\columnwidth,trim={11cm 2cm 3cm 1cm},clip]{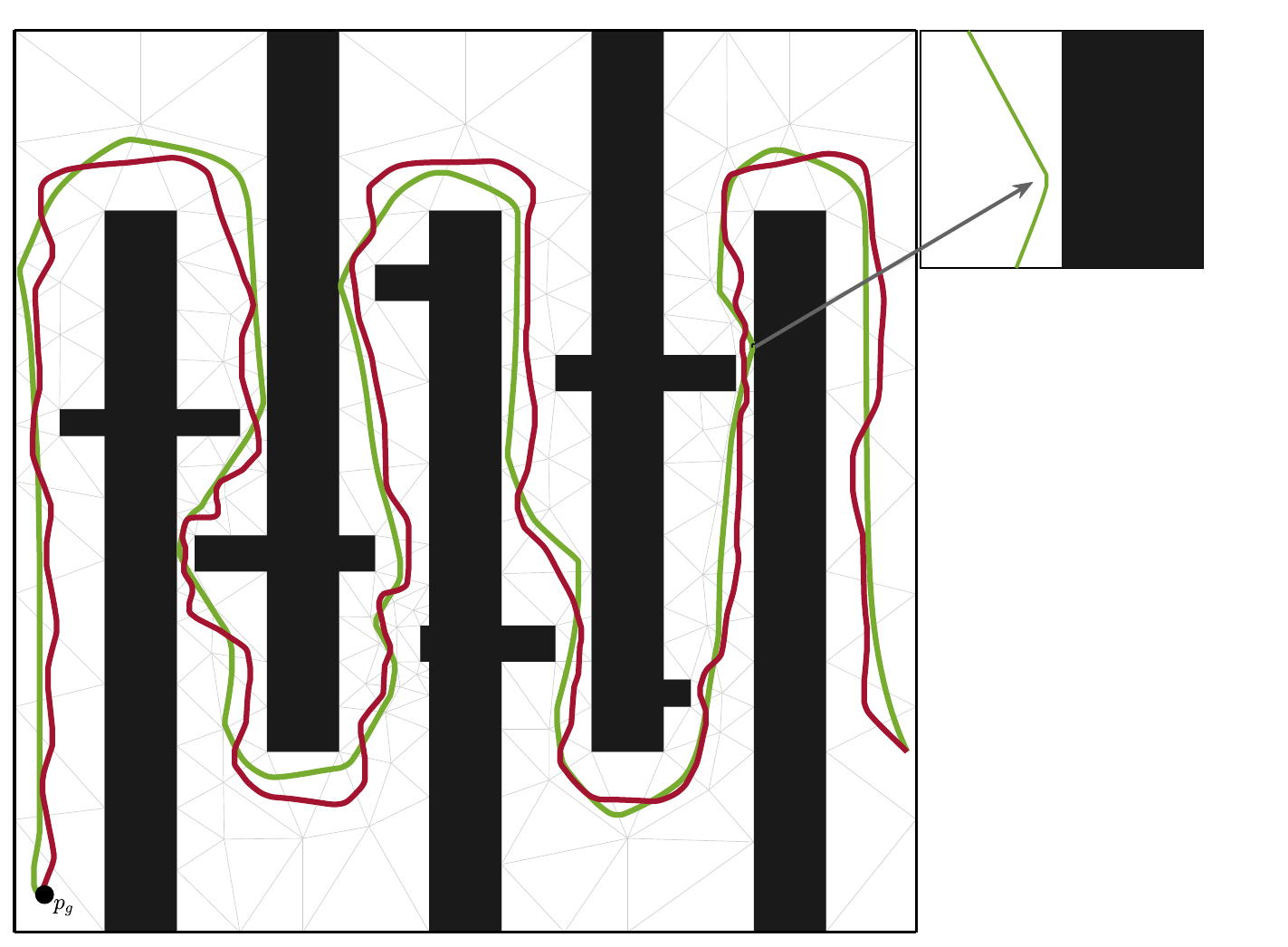} 
    \end{tabular}

    \caption{Trajectory comparison for the first experiment (left figure) and the second experiment (right figure). The green line represents the path of the robot using the proposed controller, guided by our QP-VF, and the red line represents the path of the robot guided by the baseline VF. }
    \label{fig:res1-2}
\end{figure}

\begin{figure}[t]
    \centering
\includegraphics[width=0.96\columnwidth,trim={0cm 0cm 0cm 0cm},clip]{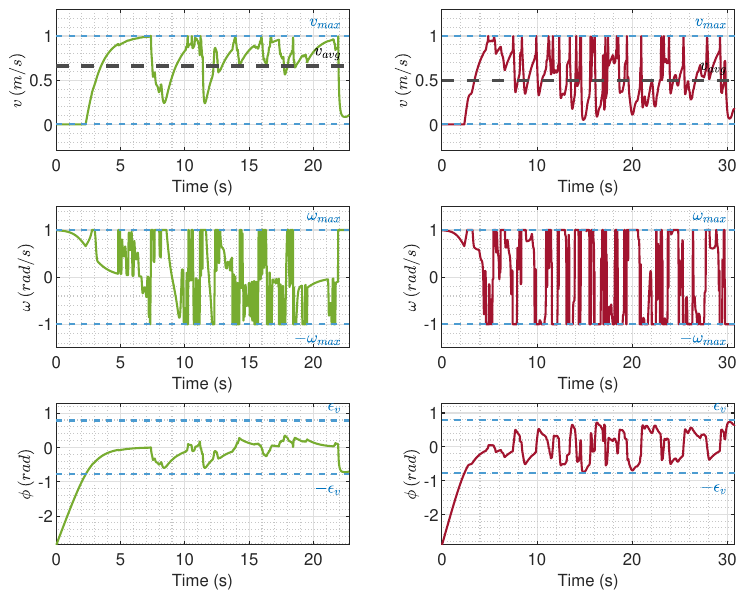}
    \caption{The results of experiment 1. The proposed QP-VF (left column) results in a higher average speed ($v_\text{avg}$, dashed line). The baseline (right column) exhibits highly saturated angular control. The bottom plots experimentally validate Proposition~\ref{prop:Nagumo}: the angular error $\phi$, which started from the unsafe set ($|\phi| > \pi/2$), is driven into the safe set and remains there.}
    \label{fig:inputres1}
\end{figure}

\begin{table}[h!]
\centering
\caption{Performance Metrics Comparison (Experiment 1)}
\label{tab:res1}
\begin{tabular}{lcc}
\toprule
\textbf{Metric} & \textbf{QP-VF} & \textbf{Baseline VF} \\
\midrule
Arrival Time (s)      &22.76& 30.77 \\
Path Length (m)        & 15.032 & 15.261 \\
Average Speed (m/s)    & 0.660  & 0.496 \\
Angular Control Effort  & 10.934 & 22.619 \\
Time Saturated (\%)    & 26.81  & 52.00  \\
\bottomrule
\end{tabular}
\end{table}
\subsection{Experiment 2: Narrow Environment with Hard Saturation}
This experiment was designed to evaluate the robustness of the proposed method under hard input constraints: $v_\text{max} = 0.8$~m/s and $\omega_\text{max} = 0.1$~rad/s. The feedback gain was set to $k=1$ and the velocity gain $\epsilon_v = \pi/5$. The robot started from a position near the obstacle, facing the obstacle. The results are summarized in Table~\ref{tab:res2} and visualized in Figure~\ref{fig:res1-2} and Figure~\ref{fig:inputres2}.

The results show that even under hard saturation, the robot could safely converge to the goal $p_g$. The high curvature of the vector fields demands an angular velocity that the robot cannot provide. This leads the controller to remain at a high rate of saturation time. As discussed in the safety analysis, the controller mitigates the resulting large orientation error by reducing the linear velocity $v$. Results show that, by the QP-VF, the robot reaches the goal nearly twice as fast with a lower angular controller effort (3.341 vs 7.326). 

\begin{figure}[t]
    \centering
\includegraphics[width=0.96\columnwidth,trim={0cm 0cm 0cm 0cm},clip]{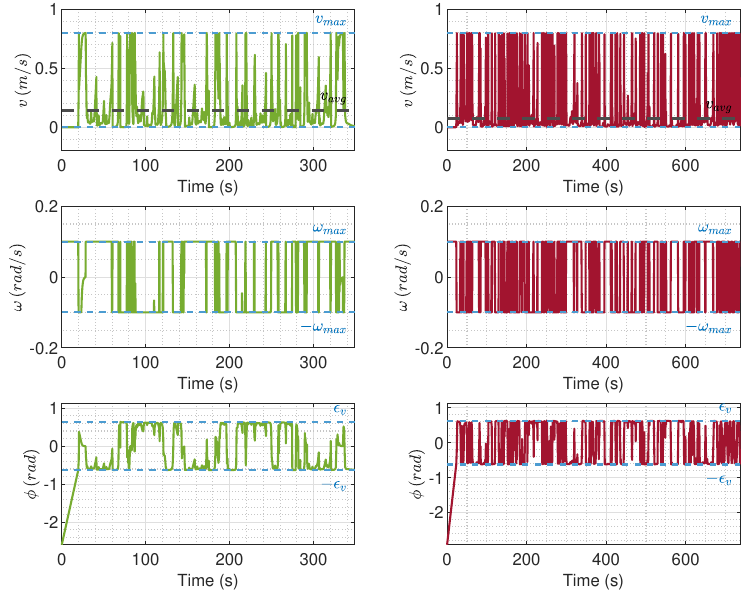}
    \caption{The results of experiment 2.
    These plots demonstrate the self-regulating loop discussed in Section~\ref{sec:Stability}. The proposed \gls{qp}-VF (left column) allows for a higher average speed ($v_\text{avg}$, dashed line) compared to the baseline (right column).}
    \label{fig:inputres2}
\end{figure}

\begin{table}[h!]
\centering
\caption{Performance Metrics Comparison (Experiment 2)}
\label{tab:res2}
\begin{tabular}{lcc}
\toprule
\textbf{Metric} & \textbf{QP-VF} & \textbf{Baseline VF} \\
\midrule
Arrival Time (s)      &348.92 & 736.70 \\
Path Length (m)        & 49.798  & 52.134  \\
Average Speed (m/s)    & 0.143   & 0.071  \\
Angular Control Effort  & 3.341  & 7.326  \\
Time Saturated (\%)    & 94.74  & 99.14   \\
\bottomrule
\end{tabular}
\end{table}

Our method splits the computation into a pre-computation phase and an online phase. The pre-computation phase includes the generation of triangulation, computation of the discrete plan, and assignment of local vector fields via \gls{qp}. The online phase includes finding the robot's current cell (if it is unknown) via point-location query techniques, finding distances to faces and blending local vector fields to evaluate $V(p)$, and finally computing control inputs.

Triangulation for a 2D polygonal environment with \gls{cdt} without Steiner Points, can be done in $O(n \log n)$, where $n$ is the number of obstacle vertices~\cite{shewchuk1996triangle}. The successor of each triangle can be found in $O(m \log m)$, in which $m$ is the number of triangles~\cite{lindemann2006real}. On a system with an AMD Ryzen 9 8945HS CPU and 32 GB RAM, triangulation using the \href{https://www.cs.cmu.edu/~quake/triangle.html}{\texttt{triangle}} library \cite{shewchuk1996triangle} took about 1 ms, and the assignment of the QP-based cell vector field required about 1 ms with the \href{https://osqp.org/docs/}{\texttt{OSQP}} solver \cite{osqp}, but the total pre-computation took about 10 ms. Triangulation in MATLAB can also be done with \href{https://github.com/dengwirda/mesh2d}{\texttt{MESH2D}} library \cite{engwirda2014locally}. While efficient for static obstacles, the framework can handle dynamic obstacles by globally recomputing the plan at about 100 Hz (for standard environments). However, the code can be further optimized for more efficiency. Online execution is lightweight. For a given state, the method performs a point-location query, which requires $O(n)$ time \cite{lindemann2006real}, followed by minimal calculation for the distances to the faces for the calculation of vector fields blending and controller evaluation.

\section{Conclusion}
We presented a framework for feedback motion planning and control for a non-holonomic unicycle robot. We introduced a novel \gls{qp} formulation that globally optimizes the assignment of cell vector fields, generating a $\Cinf$-smooth vector field with measurably lower total bending and total turning. Building upon this vector field, we designed a nonlinear controller, inspired by~\cite{lindemann2006real}, to safely navigate the robot from almost any admissible state while respecting input limits. We proved that our controller inherently satisfies safety conditions without requiring online optimization.
Simulations confirmed the effectiveness of our method. Even under hard saturation limits, the robot safely converged to the goal. The synergy between the proposed vector field and controller resulted in significantly lower arrival times and reduced angular control effort compared to the baseline. Future work could focus on extending this framework to more complex systems, such as car-like robots.


\end{document}